\documentclass[10pt,journal,final,finalsubmission,twocolumn]{article}
\catcode`\@=11

\usepackage{color}
\usepackage[cmex10]{amsmath} 
\usepackage{array}
\usepackage{fixltx2e}

\usepackage{amsfonts,amsthm,amssymb}
\usepackage{bm}
\usepackage[small,bf]{caption}
\usepackage{graphicx}
\usepackage{multirow} 
\usepackage{lineno}
\usepackage{float}
\usepackage{color}
\usepackage{lettrine}

\restylefloat{table}

\newtheorem{definition}{Definition}
\newtheorem{proposition}{\textbf{Proposition}}
\newtheorem{theorem}{\textbf{Theorem}}

\newcommand{\bmath}[1]{\mbox{ \boldmath $\!#1\!$ \unboldmath}}

\catcode`\@=11

\def\tdot#1{\smash{
\mathop{#1}\limits^{...}}}

\def\evat#1{\smash{
\mathop{ }_{#1}}}

\def\eqalign#1{\null\,\vcenter{\openup\jot \m@th
\ialign{\strut\hfil$\displaystyle{##}$&$
\displaystyle{{}##}$\hfil \crcr#1\crcr}}\,}

\def\eqalignno#1{\displ@y \tabskip=\centering
\halign to\displaywidth{\hfil$\@lign\displaystyle{##}$
\tabskip=0pt &$\@lign\displaystyle{{}##}$
\hfil\tabskip=\centering
&\llap{$\@lign##$}\tabskip=0pt\crcr #1\crcr}}

\def\leqalignno#1{\displ@y \tabskip=\centering
\halign to\displaywidth{\hfil$\@lign\displaystyle{##}$
\tabskip=0pt &$\@lign\displaystyle{{}##}$
\hfil\tabskip=\centering &\kern-\displaywidth\rlap{$\@lign##$}
\tabskip=\displaywidth\crcr #1\crcr}}

\def\eqors#1{\smash{
\mathop{=}\limits^{#1}}}

\begin{document}

\title{Global Vertices and the Noising Paradox
}


\author{Konstantinos~A.~Raftopoulos\thanks{K.Raftopoulos, raftop@image.ntua.gr,  Computer Science Division, National Technical University of Athens, Iroon Polytechneiou 9, 15780.},~Stefanos~D.~Kollias\thanks{S.Kollias, stefanos@cs.ntua.gr, Computer Science Division, National Technical University of Athens, Iroon Polytechneiou 9, 15780.},~Marin~Ferecatu\thanks{M.Ferecatu, marin.ferecatu@cnam.fr, Conservatoire National des Arts et Métiers, Laboratoire CEDRIC, Equipe Vertigo, 292 rue Saint-Martin 75003 Paris~-~France.}}

\date{}

\maketitle

\begin{abstract}
A theoretical and experimental analysis related to the identification of vertices of unknown  shapes is presented. Shapes are seen as real functions of their closed boundary. Unlike traditional approaches, which see curvature as the rate of change of the tangent to the curve, an alternative \textit{global} perspective of curvature is examined providing insight into the process of noise-enabled vertex localization. The analysis leads to a paradox, that certain vertices can be localized \textit{better} in the presence of noise. The concept of \textit{noising} is thus considered and a relevant \textit{global} method for localizing \textit{Global Vertices} is investigated. Theoretical analysis reveals that induced noise can help localizing certain vertices if combined with global descriptors. Experiments with noise and a comparison to localized methods validate the theoretical results. 
\end{abstract}

\section{Introduction}
\label{intro}
\lettrine[lines=2]{C}{urvature}, as a descriptor of shape (e.g. describing the boundary of planar shapes) possesses a rare combination of good properties: It is intrinsic, intuitive, concise, well defined, analytic, extensively studied, well understood and of an undisputed perceptual importance. However, there are two serious problems concerning its such use. 
One has to do with noise. In a noisy curve, having, that is, high frequency Fourier components (hfFc) of no perceptual importance, the local nature of curvature restricts it in describing the noise itself rather than the underlying shape. Knowing whether hfFc of a curve represent noise or not, would require solving the harder problem of recognizing the object. Since hfFc might be defining for certain shapes or just noise in others, their presence in unrecognized (unknown) shapes is considered problematic, albeit they may present useful shape information. In practice, they are usually eliminated from the boundary of all shapes, by means of a blind step of smoothing, at the risk of losing useful discriminating shape information. Smoothing also distorts the shape's metrics in an unpredictable manner, a highly undesirable effect whenever certain \textit{morphometric} measurements are defining for classification. Another problem in relation to curvature as a descriptor, has to do with \textit{meaningfulness}. Even in noise free curves, the local nature of curvature doesn't permit any kind of \textit{context} by means of which one could differentiate between points of similar curvature with respect to their perceptual characteristics on different parts of the curve. The reason behind both of these problems is curvature's local nature. Any solution would have to defy the very definition of curvature. 

The local nature of curvature is challenged in this manuscript. An attempt to address the above problems is presented, based on a \textit{global} definition of curvature that permits noise invariance and differentiation based on \textit{location} perceptual characteristics. The global definition of curvature presented herein is based on the theoretical findings of \cite{CVIU}. Curvature at a point is seen in relation to the rate of change (per infinitesimal arc length) of the point's distance, to the rest of the curve. It is thus distance based and location dependent, a fact that is further investigated in this manuscript. This global definition of curvature is shown to \textit{absorb} noise and convey \textit{context} in a way that noise becomes an \textit{enabler} for vertex identification. The new concept of \textit{noising} (as opposed to smoothing) is thus conceived and a new method for identifying vertices without even having to calculate curvature is presented. Experiments with induced noise validate the theoretical results. 

\section{Related Work}\label{REL}
Shape descriptors have been used for pattern recognition in image processing applications for many decades now \cite{rel1}. Indeed, shape is a fundamental characteristic of objects in images, widely used in object matching, identification and classification tasks. Shape descriptors divide roughly between area based methods, which use interior shape points, and boundary based methods, which use boundary information only  \cite{rosin08}. Area based descriptions (such as geometric moments, moments invariants \cite{Flusser2009,Xu2008} and shape orientation \cite{ZunicRK06}, elongation \cite{Stojmenovic08} and circularity \cite{ZunicHR10}) are less accurate and thus easier to compute and can be applied to data coming from low quality sensors. On the other hand, boundary based methods (such as Curvature Scale Space \cite{Abbasi99}, Fourier descriptors \cite{ZhangL03} and shape contexts \cite{belongie02}) are more suitable for high precision computer vision tasks and need data of good quality, data which starts to become available today, due to recent progress in sensor technology \cite{Sladoje2009hbl}.

The method presented in this work relates closer to the last category: the general approach is to describe shapes by real functions \cite{enormus} with the goal of estimating local shape features (e.g. curvature estimation, corner detection). However, numerical differentiation is sensitive to noise, thus much effort in recent years focused in finding methods that are robust to noise. In this case, one either use a smoothing pre-processing step (which usually affects the shape in a non-predictable way \cite{porteous01}), or use approximation of the real data with smooth curves \cite{GoldfeatherI04,Razdan04}.
Due to the local nature of noise, global methods usually integrate over larger scale characteristics of the data, in an attempt to inject invariance into the algorithms \cite{rel6}. Some of these approaches include geometric flows \cite {Bajaj03}, integral invariants defined via distance functions \cite{rel2}, tensor voting \cite{rel6}, active shape models using prior distributions to support large deformations in diffeomorphic metric mappings \cite{rel13} and arc detection based on arithmetic discrete lines \cite{rel6}. For the integral invariant estimators, the authors of \cite{Coeurjolly13} obtain convergence results when the grid resolution tends to zero and provide explicit formulas for the kernel size, which guarantees uniform convergence for smooth enough curves.


A multi-scale corner detector has been proposed by the authors of \cite{rel4}. It first uses an adaptive local curvature threshold instead of a single global threshold as the method in CSS \cite{ZhangL03} does and then the angles of corner candidates are checked in a dynamic region of support for eliminating falsely detected corners. In \cite{Kerautret09}, a curvature estimator is proposed that considers all the possible shapes and selects the most probable one with a global optimization approach. The estimator uses local bounds on tangent directions defined by the maximal straight segments. The estimator is adapted to noisy contours by replacing maximal segments with maximal blurred digital straight segments.
The authors of \cite{Nguyen2011}, propose a method for dominant point detection and polygonal representation of noisy and possibly disconnected curves based on the decomposition of the curve into a sequence of maximal blurred segments.

Unlike the methods described above, for which noise has a detrimental effect and thus must be resisted, bypassed or eliminated, the approach presented herein improves on this point by transforming noise into a facilitating factor for vertex identification.

\section{Contribution}
This paper serves as an extension to the Raftopoulos and Ferecatu conference paper on \textit{Noising versus Smoothing} \cite{CVPR14} extending both its theoretical and experimental findings. 

\textit{Global Vertices} introduced here for the first time, engage global shape information in the inherently local task of identifying vertices.  Due to their location at the extremes of the boundary \textit{Global Vertices} are considered \textit{undisputed} points of extreme curvature exposing this way a \textit{hidden} relation between location and curvature. 
The present paper builds on the results of \cite{CVIU} and \cite{CVPR14} adding the following elements: 
\begin{itemize}
\item
The concept of \textit{Global Vertices} and the theoretical foundation of their existence. 
\item
A formal mathematical analysis on \textit{noising} and its relation to Global Vertices. 
\item
New experiments with noising for localizing Global Vertices and the such comparison to other methods in these tasks.
\end{itemize}


 The rest of the paper is structured as follows: Previous necessary material briefly revisited in section \ref{sec:VAR}, the Algebraic Interaction of Localities and a Zero Crossing Analysis of the second GL equation are presented in sections \ref{sec:ALG} and \ref{sec:zero} respectively. The method of Noising and its use in Global Vertex Localization are presented in section \ref{sec:noising1} while Experimental Evaluation is performed in section \ref{sec:exp}. A discussion in section \ref{sec:disc} closes the paper.

\begin{figure}[t]
\begin{center}
\includegraphics[height=8.0cm]{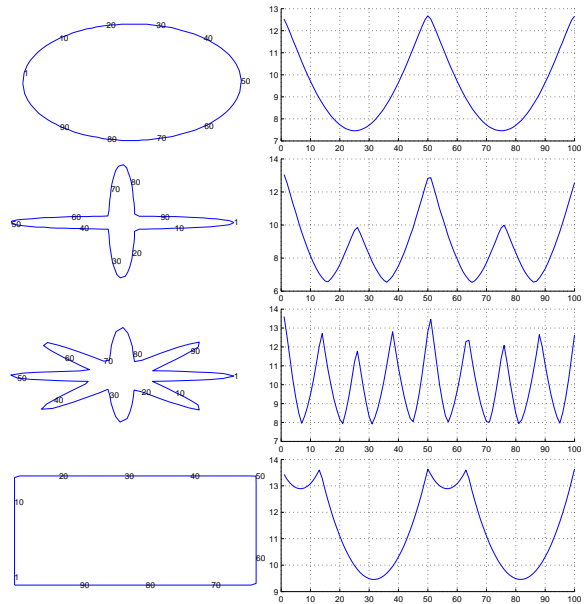}
\caption{Basic shapes (left column) and the corresponding View Area Representations (VARs, right column).}
\label{fig:allarea}
\end{center}
\end{figure} 
\section{Connection to Previous Material: The VAR descriptor and the Global-Local Equations}\label{sec:VAR}
For every point on a curve one may consider the sum of its distances to all the other points. In the discrete case this is a summation. In the continuous case it is an integral. This \textit{total distance} of a point to all the other points (or to the rest of the curve as one may choose to see it) captures a relationship between \textit{location} and curvature and this is important because location, as opposed to curvature, is not affected by noise. The VAR descriptor introduced in \cite{CVIU} is the above normalized \textit{total distance} and the \textit{Global Local (GL) Equations}\cite{CVPR14} are based on VAR to describe relations between location and curvature on the curve. In this section we briefly introduce the reader to the essentials on VAR and the GL equations. To simplify things all curves are assumed with the necessary continuity, differentiability and regularity assumptions, details on these can be found in \cite{CVIU}. These assumptions are necessary because differential geometry will be used and the mathematics have to make sense. However, these assumptions don't restrict in any way the usual reality about digital curves. Any continuous digital curve (polyline) can acquire these properties after a suitable treatment (smoothing) at its corners at an infinitesimal level that doesn't change the discrete  appearance of the curve. In Fig. (\ref{fig:allarea}) various basic shapes, with corners or/and vertices (generalization of a corner) are shown together with their VAR descriptors. Even though VAR describes the \textit{location} of the points on the curve, it is obvious that VAR's local extrema are related to corners/vertices in a \textit{global} way that is not affected by noise. The resistance to noise has been investigated in \cite{CVIU} with various experiments. In \cite{CVPR14} things advanced one step further showing that noise actually helps this representation of vertices, a concept that will be further analyzed in this manuscript.        
We proceed formally now introducing the necessary mathematical concepts that will be used for further analysis. Let $(0,\lambda\rbrack\subset\mathbb{R}$ and $\bmath{\alpha} : (0,\lambda\rbrack\rightarrow \mathbb{R} ^2$ a continuous one to one (equivalent to non self intersecting), {\it planar curve of \textit{non zero} length $\lambda$} in $\mathbb{R} ^2$, parametrized with respect to the arc length $s$ with all the necessary assumptions. 

Let us consider $\mathbb{R} ^2$ endowed with the usual metric, the plane distance, derived by the usual norm ${\Vert . \Vert}$ and the trajectory of $\bmath{\alpha}$ as $\bmath{\alpha} ((0,\lambda\rbrack)$ (the image of the interval $(0,\lambda\rbrack$ through the vector function $\bmath{\alpha}$). 

\begin{definition}
Let $\bmath{x_0}=\bmath{\alpha}(s_0)$ be a point of $\bmath{\alpha}$ and $\Delta$ the diameter\footnote{Supremum of all pair-wise distances between points of $\bmath\alpha$.} of $\bmath\alpha$. The function:
\rm
\begin{equation}
\eqalign{
v_{s_0}:(0,\lambda\rbrack\rightarrow (0,\Delta\rbrack,s\mapsto v_{s_0}(s):=\Vert\bmath{x_0}-\bmath{\alpha}(s)\Vert
}
\end{equation}
\noindent
we call the {\it view of $\bmath{\alpha}$ from the point $\bmath{x_0}\equiv\bmath{\alpha}(s_0)$}.
\end{definition}
~\\[1pt]
The \textit{view} function $v_{s_0}$ is well defined given a choice of $\bmath{x_0}$, it is also continuous as a norm in $\mathbb{R}^2$. The view functions map the arc length to the chord length of the curve. 
Lets denote with $s_*$ a certain value of the arc length parameter corresponding to a point on the curve where the unit tangent and unit normal vectors are now considered. $\dot{\bmath{\alpha}}(s_*)=\bmath{t}$ (the dot denotes the derivative with respect to the arc length $s$) is the unit vector, tangent to the curve at $\bmath\alpha (s_*)$ and $\bmath{n}$ is the normal \textit{inward} \footnote{Directed to the interior of the bounded region defined by $\bmath{\alpha}$.} unit vector also at $\bmath\alpha (s_*)$, with 
 \begin{equation}
 \eqalign{
 \frac{d\bmath t} {ds} = -\kappa(s_*)\bmath n,  \frac{d\bmath n} {ds} = \kappa(s_*)\bmath t
 }
\end{equation}
\noindent
the {\it Frenet equations} for curve $\bmath{\alpha}$ at $\bmath{p}=\bmath{\alpha}(s_*)$, with $\kappa(s_*)$ being the curvature at $\bmath{p}$, as shown in Fig. \ref{fig:Viewcut}. For the curvature $\kappa(s_*)$, we have adopted the sign convention, according to which, $\kappa(s)$ is positive when $\bmath{n}(s)$  points \textit{away} from the center of curvature.  We also consider any choice of $\xi\in (0,\lambda\rbrack, \xi\ne s_*$ and the respective point $\bmath{q}=\bmath{\alpha}(\xi)$.  We call $\bmath{r}(s_*,\xi)$ the vector $\bmath{\alpha}(s_*)-\bmath{\alpha}(\xi)$ and $\omega(s_*,\xi)$ the angle from $\bmath{n}(s_*)$ to $-\bmath{r}(s_*,\xi)$, measured counter-clockwise. Angle $\omega$ is a scalar function of two variables, $s$ and $\xi$; and $\bmath{r}$ is a vector function of the same two variables; in both cases the variables will be omitted in the notation from now on. 
\begin{proposition}\label{Prop:1}
Let $\bmath{\alpha}(s):(0,\lambda\rbrack)\rightarrow\mathbb{R}^2$ a closed, planar curve of nonzero length $\lambda$, as above. Let $\bmath{p}=\bmath{\alpha}(s_*)$ a point on the curve and $\bmath{q}=\bmath{\alpha}(\xi), \xi\ne s_*$ another point on the curve with $v_\xi(s)$ the respective view function. Then:
\begin{equation}
\label{dv4}
\eqalign{
\frac{d}{ds} v_{\xi}(s_*)=\left.-sin(\omega)\middle|\evat{s=s_*}\right.
}
\end{equation}
and
\begin{equation}
\label{dv3a}
\eqalign{
\frac{d^2v_{\xi}}{ds^2}(s_*)=\kappa(s_*)\left.cos(\omega)\middle|\evat{s=s_*}\right.  +\left.\frac{cos^2(\omega)}{\Vert\bmath{r}\Vert}\middle|\evat{s=s_*}\right.
}
\end{equation}
  
\begin{figure}
\centering
\includegraphics[height=5cm]{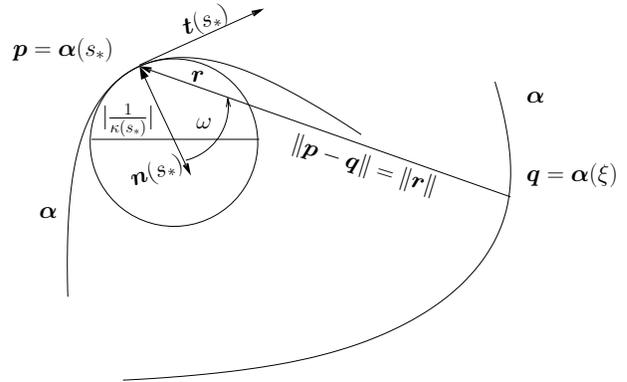}
\caption{Local bindings of the view functions. Two portions of the same boundary, the osculating circle, the Frenet frame and the view function $v_{\xi}(s_*)=\Vert\bmath{r}\Vert$.}
\label{fig:Viewcut}
\end{figure}

\end{proposition}
~\\[1pt]
This last proposition 
relates the planar curvature of $\bmath{\alpha}$ at $\bmath{p}$ to the second derivative of some view function at $s_*=\alpha^{-1}(\bmath{p})$ and we notice that, given $\bmath{p}$, the result holds for \textit{any} view function, since, given $s_*$, the choice of $\xi$ is arbitrary. This last observation permits the definition of VAR, as a generalization over the view functions as follows. 

Consider $S_\alpha$ the set of all the view functions of $\bmath\alpha$ and ${\Vert.\Vert}_v$ measuring the area below them as:
\begin{equation}
\eqalign{
{\Vert.\Vert}_v:S_\alpha\rightarrow\mathbb{R},v_s\mapsto{\Vert v_s\Vert}_v:=\int_0^\lambda v_s(\xi) d\xi
}
\end{equation}

\begin{definition}
 Let $\varphi_\alpha$ be a function defined on $(0,\lambda\rbrack$ and taking values in $\mathbb{R}$ as follows:
\begin{equation}
\eqalign{
\varphi_\alpha:(0,\lambda\rbrack\rightarrow\mathbb{R}\colon s\mapsto\varphi_\alpha(s):=\Vert v_{s}\Vert_v
}
\end{equation}
We call $\varphi_\alpha$ the {\it View Area Representation} or \textit{VAR} of the curve $\bmath{\alpha}\Box$. 
\end{definition}
~\\[1pt]
VAR is parametrized by the arc length and maps a point $\bmath{\alpha}(s)$ (by means of the arc length parameter $s$), to the area $\varphi_\alpha(s)$ below the view function $v_s$. VAR can be seen as measuring the \textit{distance} each point has from the rest of the curve, 
so it increases as we move to \textit{remote points} on the curve. 
In the form of a Theorem, we gather results from \cite{CVIU}. Dots represent derivatives always with respect to $s$. 
\begin{theorem}\label{propcurv}
Let $\bmath{\alpha}\in C^3( (0, \lambda\rbrack, \mathbb{R}^2)$ a closed planar curve of nonzero length $\lambda$, as above. If  $\varphi_\alpha (s)$ the total distance function (VAR descriptor), $\kappa (s)$ the curvature function and $s_*,\xi,\bmath{r}$ and $\omega$ as above, then:
\begin{enumerate}
\item
\begin{equation}
\label{eq:L156}
\eqalign{
 \dot\varphi_\alpha (s_*)=-\left.\int_0^\lambda sin(\omega)d\xi\middle|\evat{s=s_*}\right.
}
\end{equation}
\item 
\begin{equation}
\label{eq:L15aa}
\eqalign{
\ddot\varphi_\alpha(s_*)=\kappa(s_*)A(s_*)+B(s_*)
}
\end{equation}

where  $A(s_*)=\left.\int_0^\lambda cos(\omega)d\xi\middle|\evat{s=s_*}\right.$ and $B(s_*)=\left.\int_0^\lambda\frac{cos^2(\omega)}{\Vert \bmath{r}\Vert}d\xi\middle|\evat{s=s_*}\right.$ global shape descriptors measured at $\bmath\alpha (s_*)$. 
\item 
 If in addition, $\varphi_\alpha(s_*)$ a local extremum of $\varphi_\alpha(s)$. Then $\kappa(s_*)\ne 0$ and $A(s_*)\ne 0$ and
\begin{equation}
\label{eq:L15aaa}
\eqalign{
\kappa(s_*)=\frac{\ddot\varphi_\alpha(s_*)-B(s_*)}{A(s_*)}
}
\end{equation}
\end{enumerate}
\end{theorem}
\noindent
~\\[1pt]
\noindent
As we can see from  theorem \ref{propcurv} above, the integral descriptor $A$ has a similar interpretation with $\dot\varphi$. They both quantify a notion of global displacement of the whole curve with respect to the normal at $s_*$. Indeed, if we consider a point $\bmath{\alpha}(\xi)$ that traverses the curve, angle $\omega$ measures the angular displacement of this point with respect to the normal at $s_*$. Thus the integral $A$ (and $\dot\varphi$) can be thought of measuring the \textit{total angular displacement} of the whole curve with respect to the normal at $s_*$. 
The integral descriptor $B$ has a similar interpretation.

\section{An Algebra of Local Function Behaviors}\label{sec:ALG}

To gain further insight into the Global Local equations and what they say about the location of a vertex, we proceed in this section by introducing a algebraic formalism on local function behaviors. The purpose is to simplify the analysis of the GL equations, at least in terms of local extrema localization and magnitude but we feel it also carries a conceptual merit as an algebraic treatment of \textit{infinitesimals}, building on a combination of algebraic and analytical ideas. The focus will be on the localization of \textit{zero crossings} because of their significance in identifying vertices on curves but could be extended to any local function behavior.   

We start by extending the notion of a value of a function at a specific point to that of the \textit{set value} of the function \textit{around} this point. Thus, for some function $f$, instead of $f (s_*)$ we now have $f (\eta_{s_*})$, where $\eta_{s_*}$ is an infinitesimal neighborhood around $s_*$ and  $f (\eta_{s_*})$ an infinitesimal neighborhood around  $f (s_*)$ (continuity and differentiability conditions assumed as usual \textit{around} $s_*$). Now we distinguish the following local behaviors for the values of the function at $\eta_{s_*}$(i.d. around $s_*$):
\begin{table*}[]
\setlength{\tabcolsep}{0.70em}
\renewcommand{\arraystretch}{1.0}
\begin{tabular}{||c|c|c|c|c||}

$\oplus$&$zdc_f$&$zuc_f$&$cz_f$&$nz_f$\\\hline\hline
$zdc_g$&$zdc_{f+g}$&$zc_{f+g}(cz_{f+g})$&$zdc_{f+g}$&$nz_{f+g}$\\
$zuc_g$&$zc_{f+g}(cz_{f+g})$&$zuc_{f+g}$&$zuc_{f+g}$&$nz_{f+g}$\\
$cz_g$&$zdc_{f+g}$&$zuc_{f+g}$&$cz_{f+g}$&$nz_{f+g}$\\
$nz_g$&$nz_{f+g}$&$nz_{f+g}$&$nz_{f+g}$&$nz_{f+g}(zc_{f+g} , cz_{f+g})$\\\hline

%
\end{tabular}
\centering
\caption{Local algebra for the $\oplus$ operation.}
\label{imsrank}
\end{table*}
\begin{itemize}
\item[(i)] Zero crossing at $\eta_{s_*}$, we write it as $zc_{f(s_*)}$ when the function's $f$ \textit{order of contact} with the horizontal axis, around $s_*$, is exactly one. This has sub-cases that will be signified by $zdc$ and $zuc$, meaning \textit{zero down crossing} and \textit{zero up crossing} respectively, with the obvious interpretations whenever their distinction plays a role. We will occasionally use $zc$ to refer to both of them simultaneously.  
\item[(ii)] Constant zero at $\eta_{s_*}$, we write it as $cz_{f(s_*)}$. when the function's \textit{order of contact} with the horizontal axis in an infinitesimal neighborhood around $s_*$ is greater or equal to 2; All the following cases are $cz$ at $\eta_{s_*}$:
\begin{itemize} 
\item Local extrema equals zero at $s_*$.
\item Saddle stationary point equals zero at $s_*$.
\item Constant function equals zero at $s_*$.
\end{itemize} 
\item[(iii)] Non zero at $\eta_{s_*}$, we write it as $nz_{f(s_*)}$ when the function's order of contact with the horizontal axis in an infinitesimal neighborhood around $s_*$ is zero.
\end{itemize}

\noindent The above \textit{local behaviors} form the set $LB\equiv\{lb_i\}=\{zuc,zdc,cz,nz\}$, called the \textit{set of local function behaviors} and are mutually exclusive and collectively exhaustive in that sense.
Since the local behaviors are considered for a function at a specific point we can consider them as an $l$ functoid defined on $C^3(( 0,\lambda\rbrack, \mathbb{R})\times( 0,\lambda\rbrack$ and taking values in $LB$ thus:

\begin{equation}
\label{eq:tenta}
\eqalign{
l&:C^3(( 0,\lambda\rbrack, \mathbb{R})\times( 0,\lambda\rbrack\rightarrow LB\cr
&:(f,s)\mapsto l(f,s)\equiv l(f(\eta_s))
}
\end{equation}
the local behavior of function $f$ around $s$. Through $l$ we tentatively define a \textit{local point-wise addition} of local behaviors as:

\begin{equation}
\label{eq:tent+}
\eqalign{
&\oplus :LB\times LB\rightarrow LB:\cr
&(l(f,s),l(g,s))\mapsto l(f,s)\oplus l(g,s) \equiv (l(f+g),s)
}
\end{equation}
thus the $\oplus$ of the local behaviors of $f$ and $g$ at $s$ is the local behavior of $f+g$ at $s$. 
Similarly, a \textit{local point-wise multiplication} is defined:
\begin{equation}
\label{eq:tentx}
\eqalign{
&\otimes:LB\times LB\rightarrow LB:\cr
&(l(f,s),l(g,s))\mapsto  l(f,s)\otimes l(g,s) \equiv (l(fg),s)
}
\end{equation}
One can verify that the rules of Table \ref{imsrank} hold for $\oplus$. For the multiple cases, the extra entries in parenthesis are exceptions. These depend on the slope and the specific values of the functions at and around $s_*$ but they are special cases and will be treated accordingly. Similarly, the rules of Table \ref{imsrank1} hold for $\otimes$. There,  $f$, $g$, $fg$ are easily assumed and thus omitted for economy of space.

\noindent



So far $\oplus$ has been considered through local point-wise operations of functions around a \textit{common} saturation point. We now extend this local operation to allow cases of different saturation points for the two functions. Thus:
\begin{definition}
\begin{equation}
\label{eq:tentx1}
\eqalign{
\oplus &:LB\times LB\rightarrow LB\cr
&:(l(f,s_1),l(g,s_2))\mapsto (l(f+g),s_*)}
\end{equation}, 
where $s_*$ is the closest to $s_1$ point in $[s_1,s_2]$ for which it holds that $\vert f(s_*)+g(s_*)\vert$ is a global minimum of $\vert f(s)+g(s)\vert$, $s\in[s_1,s_2]$ and $l(f+g,s_*)$ is the local behavior of $(f+g)(s)$ around $(f+g)(s_*)\Box$
\end{definition}
~\\[1pt]
We can verify that when $s_1=s_2$ then $s_*=s_1=s_2$ and the case reduces to the point-wise behavior around a common saturation point already examined above. But the above definition extends the capability of the $\oplus$ operation to consider different saturation points for the local behavior of the two functions taking part in $\oplus$, in a way that favors localization of new zero crossings.  Indeed, it is easy to see that according to the above definition: $zc_f(s_1)\oplus zc_g(s_2)=zc_{f+g}(s_*)$ iff $[s_1,s_2]$ is an interval where $f$ and $-g$ intercept each other at $s_*\in[s_1,s_2]$. 
For the $\otimes$ operation an analogous extension is not necessary but what is important to notice about $\otimes$ is that $zc_f(s_*)\otimes nz_g(s_*)=zc_{fg}(s_*)$ and that if $s_1=s_*-\epsilon$, $s_2=s_*+\epsilon$, $\displaystyle\lim_{s\to s_1}f(s)=\pm\infty$, $\displaystyle\lim_{s\to s_2}f(s)=\mp\infty$ and $g$ bounded in $(s_1,s_2)$ then $\displaystyle\lim_{s\to s_1}fg(s)=\pm\infty$ and $\displaystyle\lim_{s\to s_2}fg(s)=\mp\infty$  (or $\displaystyle\lim_{s\to s_1}fg(s)=\mp\infty$ and $\displaystyle\lim_{s\to s_2}fg(s)=\pm\infty$ depending on the sign of $g$ around $s_*$).
%
%
%
In words this last one states that if the absolute values of the zero crossing participant in the $\otimes$ operation increase without bound around  $s_*$ but the non zero participant is bounded around $s_*$, then the resulting from the $\otimes$ operation behavior  is a zero crossing with its absolute values increasing without bound around $s_*$. This observation is important because it affects the localization of the resulting from $\oplus$ zero crossing as the following Proposition states:
\begin{proposition}
\label{basicprop}
If $l(f,s_*)$ is a zero crossing ($zc$) of $f$ at $s_*$ and $l(g, s)$ a random local behavior ($lb$) of  $g$ at $s$, also $s_1=s_*-\epsilon$, $s_2=s_*+\epsilon$ such that $\displaystyle\lim_{s\to s_1}f(s)=\pm\infty$ and $\displaystyle\lim_{s\to s_2}f(s)=\mp\infty$, $f$ continuous in $(s_1,s_2)$, $g$ continuous and bounded in $(s_1, s_2)$ then:$\displaystyle\lim_{s_1\to s_*}(zc_f(s_*)\oplus lb_g(s_1))=\displaystyle\lim_{s_2\to s_*}(zc_f(s_*)\oplus lb_g(s_2))=zc_{f+g}(s_*).$
\end{proposition}
\begin{proof}
We first pick the closest to $s_1$ point $s_0\in (s_1,s_2)$ such that $f(s_0)=-g(s_0)$. Such a point always exists since the  values of $f$ increase/decrease inversely without bound near $s_1$ and $s_2$ and $g$ bounded in $(s_1,s_2)$. At this point $s_0$, $f+g$ has a zero crossing but as $s_1\to s_*$ and $s_2\to s_*$ then also $s_0\to s_*$ from either side and the proof is complete. 
\end{proof}
In words this proposition states that as the absolute values of $f$ increase around $s_*$ then the location at which $f\oplus g$ achieves a zero crossing approaches $s_*$ and it is a key proposition since it permits elimination of the bounded member in $\oplus$ (in terms of zero crossing localization), when the conditions of high values around the zero crossing (e.g. by means of the observation regarding $\otimes$ above) hold for the unbounded member. 

\section{Zero Crossing Analysis of the second Global Local Equation}\label{sec:zero}

Lets visit again equation (\ref{eq:L15aaa}) which links $\kappa$ to $\ddot\phi$ up to global descriptors $A$ and $B$. All the quantities on the right hand side are integrals defined on the whole of the shape. They don't change significantly with noise. This definition of curvature is thus \textit{stronger than the traditional one}. In Fig. (\ref{fig:ViewCurv}) we show the integral descriptors $A$, $B$ and $\ddot\phi$ involved in equation (\ref{eq:L15aaa}) for smoothed and noisy versions of Kimia silhouettes \cite{kimia}. 
 
While, as we will see in the experiments, this equation permits calculating $\kappa$ even in noisy conditions, a concern here will be whether it can be further simplified, if we are \textit{only} looking for vertices. The motivation for such a concern comes from the observation that in equation (\ref{eq:L15aaa}), $A$ and $B$ may contribute to the calculation of $\kappa$ but \textit{sudden changes} (per infinitesimal arc length) of $\kappa$ (e.g. vertices) could possibly be localized by only examining the \textit{sudden changes} of $\ddot\phi$.  
If further simplification of equation (\ref{eq:L15aaa}) can lead to a faster/simpler method of estimating vertices is significant because vertices summarize the perceptual importance of curvature, while at the same time traditional local methods cannot compute vertices in the presence of noise. This exploration will lead us to a surprising result, as we will discover that the rate of change of $A$ contributes a \textit{location} aspect to the calculation of vertices. Thus the quest for vertices becomes a quest for specific locations on the curve in a way that noise is turned into a facilitating factor that permits better localization. The concept of \textit{noising} is motivated by this observation. 
With the tools for local analysis presented in the previous section at hand and assuming usual valid continuity and differentiability conditions, we start with the derivative of $A$ at the origin of the Frenet frame $s_*$: 


\begin{equation}
\label{eq:dcos}
\eqalign{
\dot A(s_*)&=\left.\frac{d}{ds}\left(\int_0^\lambda cos(\omega)d\xi\right)\middle| \evat{s=s_*} \right.\cr
&=\left.\int_0^\lambda \frac{\partial}{\partial s} cos(\omega) d\xi\middle|\evat{s=s_*}\right.\cr
&=\left.-\int_0^\lambda \dot\omega sin(\omega) d\xi\middle|\evat{s=s_*}\right.
 }
\end{equation}

From the combination of   (\ref{dv4}) and (\ref{dv3a}) we get:
\begin{equation}
\label{eq:L12}
\eqalign{
-\frac{\partial}{\partial s} sin(\omega)&=-\dot\omega cos(\omega)=cos(\omega) \left(\kappa(s)+\frac{cos(\omega)}{\Vert\bmath{r}\Vert}\right)
}
\end{equation}

and thus for $cos(\omega)\ne0$ we get: 
\begin{equation}
\label{eq:L13}
\eqalign{
-\dot\omega=\kappa(s)+\frac{cos(\omega)}{\Vert\bmath{r}\Vert}
}
\end{equation}

Substituting $\dot\omega$ from above to (\ref{eq:dcos}):
\begin{equation}
\label{eq:L20}
\eqalign{
&\dot{A}(s_*)=\kappa(s_*)\left(\int_0^\lambda sin(\omega)d\xi\middle|\evat{s=s_*}\right)+\cr
&\left.\int_0^\lambda\frac{sin(\omega)cos(\omega)}{\Vert \bmath{r}\Vert}d\xi\middle|\evat{s=s_*}\right.
}
\end{equation}

and from equation (\ref{eq:L156}):
\begin{equation}
\label{eq:L20a}
\eqalign{
\dot{A}(s_*)&=-\kappa(s_*)\dot\phi (s_*)+\left.\int_0^\lambda\frac{sin(\omega)cos(\omega)}{\Vert \bmath{r}\Vert}d\xi\middle|\evat{s=s_*}\right.
}
\end{equation}

\begin{figure*}[t]
\makebox[\textwidth][c]{
\begin{minipage}[b]{8cm}
\includegraphics[height=8cm]{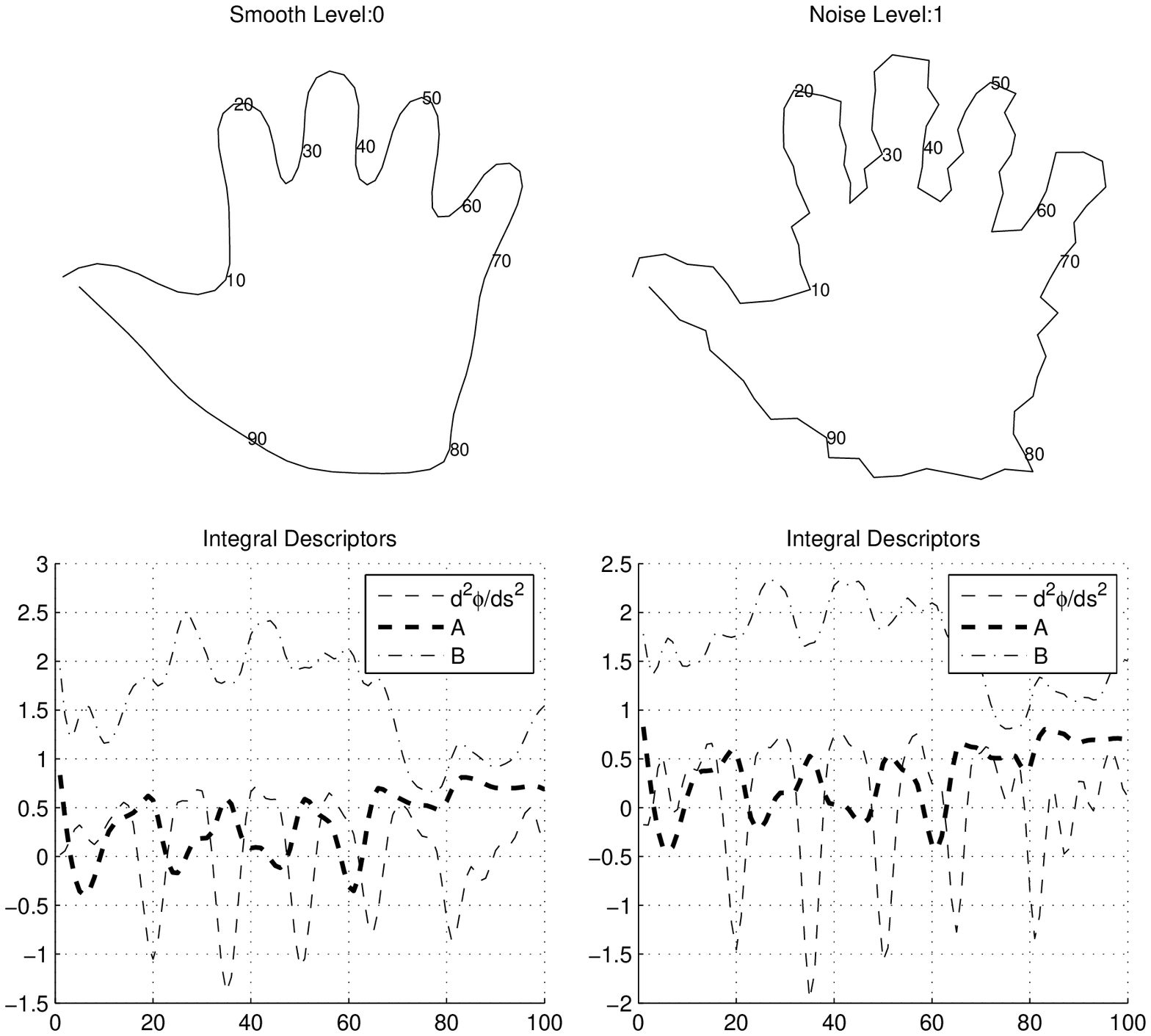}
\end{minipage}
\begin{minipage}[b]{8cm}
\includegraphics[height=8cm]{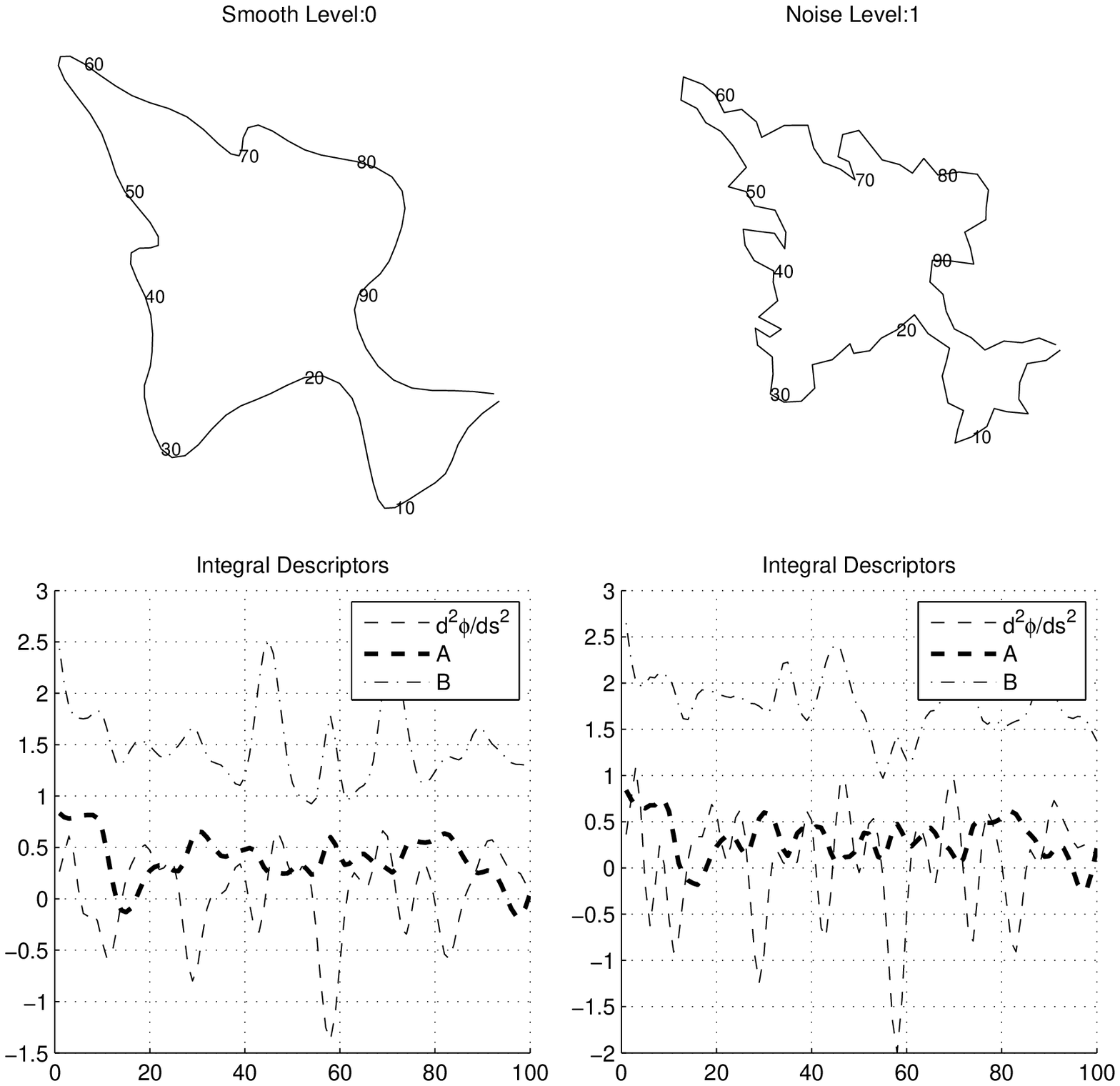}
\end{minipage}
}
\caption{The integral descriptors $A$, $B$ and $\ddot\phi$, in noisy and smoothed versions of Kimia silhouettes. Each shape is sampled by 100 equal spaced points marked in steps of 10 in the Figure and depicted in the x-axis of the corresponding plots. Notice there is no significant distortion due to noise. Curvature can be defined through these descriptors in a global manner.}
\label{fig:ViewCurv}
\end{figure*}
This last equation reveals an unexpected behavior of $\dot A$. It states that local extrema locations of $A$ and $\phi$ tend to each other as $\vert\kappa\vert$ increases. To see this we need to recall Proposition \ref{basicprop} and notice that the integral at the far right side of equation \ref{eq:L20a} is constant with respect to $\kappa$. According then to Proposition \ref{basicprop}, $\dot A$ achieves a zero crossing as close to a zero crossing of  $\dot\phi$ as large is $\vert\kappa\vert$. In other words, the location at which $\dot A$ achieves a zero crossing can be arbitrarily close to a zero crossing $\dot\phi (s_*)$ given large enough $\vert\kappa (s_*)\vert$. This result is encouraging because it links $\dot A$ with $\dot\phi$ in terms of zero crossing analysis and thus may lead to a further simplification of equation \ref{eq:L20a}, as was initially pursued. We continue in hopes that similar conclusions hold for $\dot B$:
\begin{equation}
\label{eq:dcos1}
\eqalign{
&\dot B(s_*)=\left.\frac{d}{ds}\left(\int_0^\lambda \frac{cos^2(\omega)}{\Vert r\Vert}d\xi\right)\middle| \evat{s=s_*} \right.\cr
&=\left.\int_0^\lambda \frac{\partial}{\partial s} \frac{cos^2(\omega)}{\Vert r\Vert} d\xi\middle|\evat{s=s_*}\right.\cr
&=\left.-2\int_0^\lambda\frac{\dot\omega cos(\omega)sin(\omega)}{\Vert r\Vert}d\xi\middle|\evat{s=s_*}\right. +\cr
&+\left.\int_0^\lambda\frac{cos^2(\omega)sin(\omega)}{\Vert r\Vert^2} d\xi\middle|\evat{s=s_*}\right.\cr
&\eqors{(\ref{eq:L12})}2\kappa(s_*)\left(\int_0^\lambda\frac{cos(\omega)sin(\omega)}{\Vert r\Vert}d\xi\middle|\evat{s=s_*}\right) +\cr
&+\left.2\int_0^\lambda\frac{cos^2(\omega)sin(\omega)}{\Vert r\Vert^2} d\xi\middle|\evat{s=s_*}\right.+\cr
&+\left.\int_0^\lambda\frac{cos^2(\omega)sin(\omega)}{\Vert r\Vert^2} d\xi\middle|\evat{s=s_*}\right.\cr
&=2\kappa(s_*)\left(\int_0^\lambda\frac{cos(\omega)sin(\omega)}{\Vert r\Vert}d\xi\middle|\evat{s=s_*}\right) +\cr
&+\left.3\int_0^\lambda\frac{cos^2(\omega)sin(\omega)}{\Vert r\Vert^2} d\xi\middle|\evat{s=s_*}\right.
 }
\end{equation}
$\dot B$ doesn't seem to follow a similar pattern but we proceed further anyways. Starting from (\ref{eq:L15aa}) and assuming valid continuity and differentiability conditions, we take derivatives on both sides with respect to $s$:
\begin{equation}
\label{eq:2curv4}
\eqalign{
\tdot\phi_\alpha(s_*)=A(s_*)\dot\kappa(s_*)+\kappa(s_*)\dot{A}(s_*)+\dot{B}(s_*)
 }
\end{equation}
and by means of equations (\ref{eq:L20a}) and (\ref{eq:dcos1}):
\begin{equation}
\label{eq:dcos2}
\eqalign{
&\tdot\phi (s_*)=\dot\kappa(s_*)A(s_*)-\kappa^2(s_*)\dot\phi (s_*)+\cr
&+3\kappa(s_*)\left(\int_0^\lambda\frac{cos(\omega)sin(\omega)}{\Vert r\Vert}d\xi\middle|\evat{s=s_*}\right)+\cr
&+\left.3\int_0^\lambda\frac{cos^2(\omega)sin(\omega)}{\Vert r\Vert^2} d\xi\middle|\evat{s=s_*}\right.
 }
\end{equation}
We realize that if we rewrite \ref{eq:dcos2} as:
\begin{equation}
\label{eq:oplusA}
\eqalign{
&\tdot\phi (s_*)=\dot\kappa (s_*)A(s_*)\cr
&-\kappa (s_*)\left(\kappa (s_*)\dot\phi (s_*)-3\left.\int_0^\lambda\frac{cos(\omega)sin(\omega)}{\Vert r\Vert}d\xi\middle|\evat{s=s_*}\right.\right)\cr
&+\left.3\int_0^\lambda\frac{cos^2(\omega)sin(\omega)}{\Vert r\Vert^2} d\xi\middle|\evat{s=s_*}\right.
}
\end{equation}
we can proceed with a zero crossing analysis according to the theory established in section \ref{sec:ALG}.
Indeed at a zero crossing $\dot\phi (\eta_{s_*})$ and as $\vert\kappa(s_*)\vert\rightarrow\infty$, equation (\ref{eq:oplusA}) can be now examined at a zero crossing basis as:
\begin{equation}
\label{eq:oplusB}
\eqalign{
&\tdot\phi (\eta_{s_*})=\dot K(\eta_{s_*})\otimes A(\eta_{s_*})\oplus -K (\eta_{s_*})\cr
&\otimes\left[\left(K(\eta_{s_*})\otimes\dot\phi (\eta_{s_*})\right)\oplus -3\left.\int_0^\lambda\frac{cos(\omega)sin(\omega)}{\Vert r\Vert}d\xi\middle|\evat{s\in\eta_{s_*}}\right.\right]\cr
&\oplus\left.3\int_0^\lambda\frac{cos^2(\omega)sin(\omega)}{\Vert r\Vert^2} d\xi\middle|\evat{s\in\eta_{s_*}}\right.
}
\end{equation}
with $K(\eta_{s_*})$ meaning \textit{large} $\kappa$ at a neighborhood of $s_*$. At this point the reader must recall that equation (\ref{eq:oplusB}) is describing operations between \textit{local behaviors} and not between values. 
We notice that the operation in the square brackets is the case of Proposition (\ref{basicprop}).
 Indeed as $\vert\kappa (s_*)\vert\rightarrow\infty$ the absolute values \textit{around} zero of the zero crossing signified by $K(\eta_{s_*})\otimes\dot\phi (\eta_{s_*})$ increase without bound while the other term $\left.3\int_0^\lambda\frac{cos(\omega)sin(\omega)}{\Vert r\Vert} d\xi\middle|\evat{s\in\eta_{s_*}}\right.$ is bounded in $\eta_{s_*}$.  In terms of zero crossing localization, the integral term can be omitted since, as already explained, the zero crossing location is dominated by the local term, thus:  
\begin{equation}
\label{eq:oplus}
\eqalign{
&\tdot\phi (\eta_{s_*})=\left[\dot{K}(\eta_{s_*})\otimes A(\eta_{s_*})\right]\oplus\cr
&\left[-K^2(\eta_{s_*})\otimes\dot\phi (\eta_{s_*})\oplus
\left.3\int_0^\lambda\frac{cos^2(\omega)sin(\omega)}{\Vert r\Vert^2} d\xi\middle|\evat{s\in\eta_{s_*}}\right.\right]
}
\end{equation}
\noindent
and for the same reasons finally:
\begin{equation}
\label{eq:last}
\eqalign{
\tdot\phi (\eta_{s_*})&=\left[\dot K(\eta_{s_*})\otimes A(\eta_{s_*})\right]
\oplus\left[-K^2(\eta_{s_*})\otimes\dot\phi (\eta_{s_*})\right]
}
\end{equation}

\begin{table}[]
\setlength{\tabcolsep}{0.70em}
\renewcommand{\arraystretch}{1.0}
\begin{tabular}{||c|c|c|c|c||}

$\otimes$&$zdc$&$zuc$&$cz$&$nz$\\\hline\hline
$zdc$&$cz$&$cz$&$cz$&$zc$\\
$zuc$&$cz$&$cz$&$cz$&$zc$\\
$cz$&$cz$&$cz$&$cz$&$cz$\\
$nz$&$zc$&$zc$&$cz$&$nz$\\\hline

%
\end{tabular}
\centering
\caption{Local algebra for the $\otimes$ operation.}
\label{imsrank1}
\end{table}
This last equation is rewarding as it describes a localization mechanism based on the interaction between zero crossings of two crucial functions: $\dot K$ which defines vertices of significant curvature on the curve, and $\dot\phi$, $\tdot\phi$ which in combination define points of extreme location and robust curvature. Vertices are thus linked to location on the curve in a robust way through zero crossing localization without the need for curvature calculations. The restriction for \textit{high} curvature even though necessary for simplifying analysis it is also a plausible assumption, if we think that in strictly mathematical terms, vertex can be a point where curvature achieves a local extremum even if this extremum is arbitrarily close to zero on either side. It is apparent that these vertices of \textit{low} absolute curvature do not carry the desired perceptual importance. We will now show that the requirement for large $\kappa$ can be enforced to any curve in a principled way without changing its curvature and vertex characteristics.  

\section{Global Vertices and the Concept of Noising}\label{sec:noising1}
The achievement in equation (\ref{eq:last}) is that the problem of solving a differential equation almost \textit{everywhere}, is turned into that of localizing zero crossings. This is not only easier (given the tools in section (\ref{sec:ALG})) but in our case, it also has a greater practical value, because vertex analysis is now turned into a search for local extrema locations in scalar functions which is an easier task unaffected by noise. But there is another \textit{unexpected} benefit in the formulation of equation (\ref{eq:last}). The requirement for large $\kappa$ can be met by the presence of noise and thus noise is turned into a facilitating factor for localizing vertices according to equation (\ref{eq:last}). First we show that the location of the curve's vertices of significant curvature (big $\vert\kappa\vert$) can be inferred from the extrema of $\phi$. This is important because as we already discussed, $\phi$ is a scalar function and its extrema are easily identifiable even after extensive noise on the curve. We proceed formally with the following Proposition that also serves as the definition of a \textit{Global Vertex}:
 
\begin{proposition}\label{propvertex1}
Let $\bmath{\alpha}\in C^3(( 0, \lambda\rbrack, \mathbb{R}^2)$ a closed planar curve of length $\lambda\ne 0$, at least 3 times differentiable as a function defined on the unit circle, $\phi_\alpha (s)$ the total distance function (VAR descriptor), and $\kappa (s)$ the curvature function. If $\phi_\alpha (s)$ achieves a local extremum in $\eta_{s_*}$ in such a way that $\dot\phi_\alpha (\eta_{s_*})$ is a zero crossing of $\dot\phi_\alpha (s)$ and $\tdot\phi_\alpha (\eta_{s_*})$ a zero crossing of $\tdot\phi_\alpha (s)$ then at large curvature values $K(s_*)$ 
there exists a vertex of $\bmath\alpha$ in $\bmath\alpha (\eta_{s_*})$. 
%
\end{proposition}
\begin{proof}

Since  $\phi_\alpha(s_*)$ is an extreme of $\phi_\alpha (s)$ in a way that both $\tdot\phi_\alpha(\eta_{s_*})$ and  $\dot\phi_\alpha(\eta_{s_*})$ are zero crossings of $\tdot\phi_\alpha(s)$ and $\dot\phi_\alpha(s)$ respectively, 
from Theorem (\ref{propcurv}) we get that $\kappa (s_*)\ne0$ therefore $K^2(\eta_{s_*})\otimes\dot\phi_\alpha(\eta_{s_*})$ is a zero crossing. This according to equation (\ref{eq:last}) means that $\dot K(\eta_{s_*})\otimes A(\eta_{s_*})$ will have to allow one of the following cases: 
	  \begin{itemize}
	  \item $\dot K(\eta_{s_*})\otimes A(\eta_{s_*})$ is also a zero crossing. But since from Theorem (\ref{propcurv}), $A(s_*)\ne0$ we conclude that $\dot K(\eta_{s_*})$ is a zero crossing and in $\bmath{\alpha}(\eta_{s_*})$ there is a vertex of $\bmath{\alpha}$.
	  \item $\dot K(\eta_{s_*})\otimes A(\eta_{s_*})$ is locally zero. But again from Theorem (\ref{propcurv}), $A(s_*)\ne0$ thus $\dot K(\eta_{s_*})$ is locally zero. Since also from Theorem (\ref{propcurv}) $\kappa (s_*)\ne0$, we conclude that $\bmath{\alpha}(s)$ has locally constant non zero curvature at $\eta_{s_*}$ thus in $\bmath{\alpha}(\eta_{s_*})$ there is a vertex of $\bmath{\alpha}$.  
	   \end{itemize}
So in both of the above cases in $\bmath{\alpha}(\eta_{s_*})$ there is a vertex of $\bmath{\alpha}$ and the proof is complete\end{proof}
\begin{figure}[]
\begin{center}
\includegraphics[height=13cm]{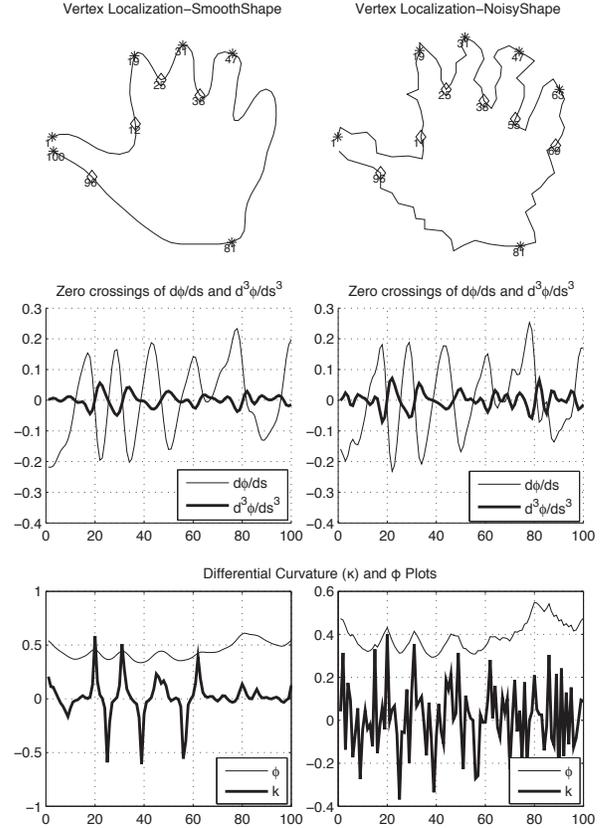}
\caption{\textit{Global Vertex} localization for smooth and noise versions of the same Kimia Silhouette using co- localization of  $\phi$ and $\ddot\phi$ extreme points. Stars and diamonds are curvature's local maxima and minima respectively. Points are marked on the shapes for every 10th point in a total of 100 points per shape and are also assumed as the x axis in all plots. The co-localization of zero crossings that appear in the second row of plots are validated against $\phi$ and curvature plots appearing in the last row. We notice that more points are correctly identified in the noisy version. We also notice that the proposed method produces correct results even though the differential curvature descriptor has collapsed in the noisy case.}
\label{fig:V-V3}
\end{center}
\end{figure} 
\noindent
~\\[1pt]
Vertices identified according to Proposition \ref{propvertex1} will be called \textit{Global Vertices} 
and are localized \textit{natively}, not relying that is, on smooth curvature calculations. 
As Proposition (\ref{propvertex1}) suggests, a method of identifying \textit{Global Vertices}, points of extreme location and curvature in the \textit{collocation} of $\dot\phi$ and $\tdot\phi$ zero crossings, is in place. Under this method one wonders what the effect of noise might be. The unexpected answer is that the requirement for big $\kappa$ (in the absolute sense) is satisfied at all the curve points by the existence of noise. Noise, therefore is what makes equation (\ref{eq:last}) valid at all points on the curve. On a smooth curve, big $\kappa$s will exist only at the few high curvature points, the method therefore is valid for smooth curves as well but will identify vertices only if  big $\kappa$s are well located \textit{globally} on the curve. One can then wonder how much noise is enough to make equation (\ref{eq:last}) valid and weather the fact that big $\kappa$s are everywhere (after existing or induced noise on the curve) obscures the underlying shape. In Fig.(\ref{fig:ncurv}) we demonstrate a discrete case where additive noise induced in a principled manner increases the curvature of the resulting curve at all of its points, but at the same time retains the initial curvature information of the initial curve, while also enriches the tangent directions around the initial points. Thus the concept of noising is conceived. The continuous case can be similarly inferred if instead of a piecewise curve one considers infinitesimal arc elements in a suitable manner. 

\begin{figure}[]
\begin{center}
\includegraphics[height=5cm]{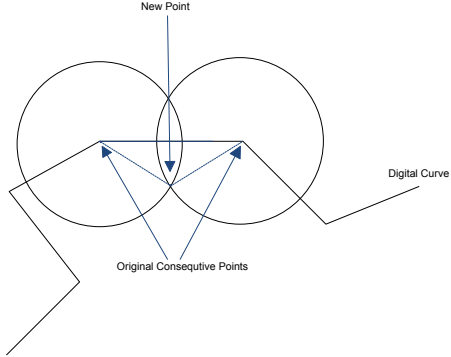}
\caption{The method of noising as opposed to smoothing is illustrated here. New points are added to the contour with the purpose to enrich the tangent directions around the original points. This will facilitate the proposed method in identifying vertices.}
\label{fig:putnoise}
\end{center}
\end{figure} 

\subsection{Finitesimals: Infinitesimals in the Large}
The amount of noise needed to make equation (\ref{eq:last}) valid can be inferred from equations (\ref{eq:oplusB}) and (\ref{eq:oplus}) where the terms $C=3\left.\int_0^\lambda\frac{cos(\omega)sin(\omega)}{\Vert r\Vert}d\xi\middle|\evat{s\in\eta_{s_*}}\right.$ and $D=3\left.\int_0^\lambda\frac{cos^2(\omega)sin(\omega)}{\Vert r\Vert ^2}d\xi\middle|\evat{s\in\eta_{s_*}}\right.$ have to be compensated by the slope of $\dot\phi(\eta_{s_*})$ and the magnitude of $K(s_*)$ and $K^2(s_*)$, respectively. The continuity of all functions involved in the above analysis guaranties a smooth transition of zero crossings as $\kappa$ increases against $C$ and $D$ and permits a valid extension from infinitesimal to finite neighborhoods around $s_*$ while keeping the validity of all the previous results. We thus permit $\eta_{s_*}$ to aquire finite \textit{length} and become $H_{s_*}$, signifying a compact interval of finite nonzero length around $s_*$ such that $\dot\phi(H_{s_*})$ and $A(H_{s_*})$ exhibit the same local behavior with $\dot\phi(\eta_{s_*})$ and $A(\eta_{s_*})$ respectively around $s_*$. We will use the term \textit{finitesimal} to refer to $H_{s_*}$ as an extension of the infinitesimal neighborhood $\eta_{s_*}$ such that its image through $\dot\phi$ and $A$ retains the same local behavior in the sense of section \ref{sec:ALG}. We can now see again equation (\ref{eq:last}) as:
\begin{equation}
\label{eq:last1}
\eqalign{
\tdot\phi (H_{s_*})&=\left[\dot \kappa_{_{H_{s_*}}}\otimes A(H_{s_*})\right]
\oplus\left[-{\kappa^2}_{_{H_{s_*}}}\otimes\dot\phi (H_{s_*})\right]
}
\end{equation}
\noindent
which is always valid as long as $\kappa$ is a function of the length of $H_{s_*}$. Indeed one can think in an inverse manner and ask that given $s_*$ (as a zero crossing location of $\dot\varphi$), $C$, $D$ and $H_{s_*}$ as a finite neighborhood around $s_*$, what is the minimum $\kappa$ that would guarantee a zero crossing of $\tdot\varphi$ in $H_{s_*}$? The answer to this question is a finite curvature, thus with the introduction of \textit{finitesimals} the requirement of a large $K(s_*)$ that goes to infinity in order of the zero crossings to coexist in an infinitesimal neighborhood $\eta_{s_*}$, has been relaxed to that of a \textit{large enough} curvature $\kappa_{_{H_{s_*}}}$ that would guaranty the existence of both zero crossings of $\dot\phi$ and $\tdot\phi$ in $H_{s_*}$ and thus, according to Proposition (\ref{propvertex1}), a vertex of $\bmath\alpha$, in the \textit{finitesimal} neighborhood $H_{s_*}$. The \textit{finitesimal} neighborhood $H_{s_*}$ is conceptualized in relation to the \textit{infinitesimal} neighborhood $\eta_{s_*}$ and the local behavior of the pertinent functions at $\eta_{s_*}$.  In the concept of $H_{s_*}$, the \textit{infinitesimal} neighborhood $\eta_{s_*}$ acquires the \textit{minimum} size necessary to meet certain constrains of the specific implementation while the pertinent functions keep the same local behavior there. This is in fact a practical definition that also captures the implementation details of noising. 
At this point a practical conclusion is also reached. Since in $C$ and $D$ we have a measure of the amount of perturbation ($\kappa_{H_*}$) needed at $s_*$ to bring in equation (\ref{eq:last1}) one could see this equivalently, as the tolerance in the \textit{collocation} of $\dot\varphi$ and $\tdot\varphi$ one could allow for the existence of a vertex in $H_*$.

\subsection{Noise as an Enabler}
The following \textit{noising} algorithm is now designed, consisting of controlled perturbations on the boundary of noisy or smooth shapes.
This process can be performed in an additive manner to the existing boundary not affecting the initial boundary points. In the discrete case of a digital curve, for each pair of consecutive points on the initial boundary a \textit{new point} will be added at the intersection of the circles centered at the original points and having equal radii of a certain length greater than half the distance between the two original points (Fig.\ref{fig:putnoise}). 
Noise applied in this principled way has three global consequences all beneficial to the proposed global method (Fig.\ref{fig:ncurv}):
\begin{itemize}
\item The absolute local curvature at each point on the curve increases, thus equation (\ref{eq:last1}) can be appropriately applied.
\item The relative rate of change of each point's \textit{total distance} to the rest of the curve is not changed significantly, thus  \textit{global curvature} calculations are also not affected.
\item The tangent directions around each point are enriched. This helps $\dot\phi$ to better identify remote points on the curve and $\tdot\phi$ to better measure the relative change of their distance to the rest of the curve. Native Global Vertex identification according to Proposition (\ref{propvertex1}) is thus facilitated.
\end{itemize}

\begin{figure}[]
\begin{center}
\includegraphics[height=3.9cm]{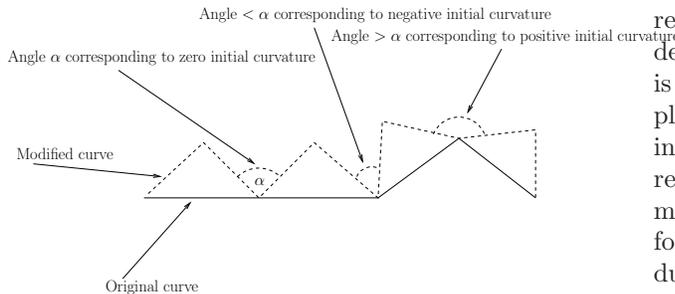}
\caption{The method of additive noising is illustrated here. New points are added to the contour but the initial curvature information is retained.}
\label{fig:ncurv}
\end{center}
\end{figure} 

One can then realize that the \textit{negative local} effects of noise have been turned into \textit{positive global} effects. For the proposed global method in particular, noise enables equation (\ref{eq:last1}) and thus the identification of vertices according to Proposition (\ref{propvertex1}). 
This emerges as a paradox, since vertices are differentials of a higher order than curvature, thus even more sensitive to noise than curvature is with traditional methods. 


It is important to notice that according to this method, Global Vertices are detected directly, by the collocation of the zero crossings as above, without the need for curvature calculations. In other words, the method does not rely on a smooth estimation of curvature around vertices but it is a \textit{native} method for calculating Global Vertices directly, albeit them being of a higher differential order than curvature. Another issue worth noticing, is that the above procedure of noising can be applied recursively to form neighborhoods of increasing differential order \textit{around} the initial curve points, resulting in an analogous concept to that of incremental smoothing. \textit{Noising}, as a preprocessing step for global methods, may be viewed as a conceptual duality to what smoothing is for local methods. 

In the experimental section that follows, the above methods for curvature calculation (according to (\ref{eq:L15aaa})) and vertex localization (according to Proposition (\ref{propvertex1})) are validated. A comparison to LAII in these tasks, for smooth and noisy versions of shapes is also presented.

\section{Experimental Validation}\label{sec:exp}

\begin{figure}[]
\begin{center}
\includegraphics[height=13cm, width=9cm]{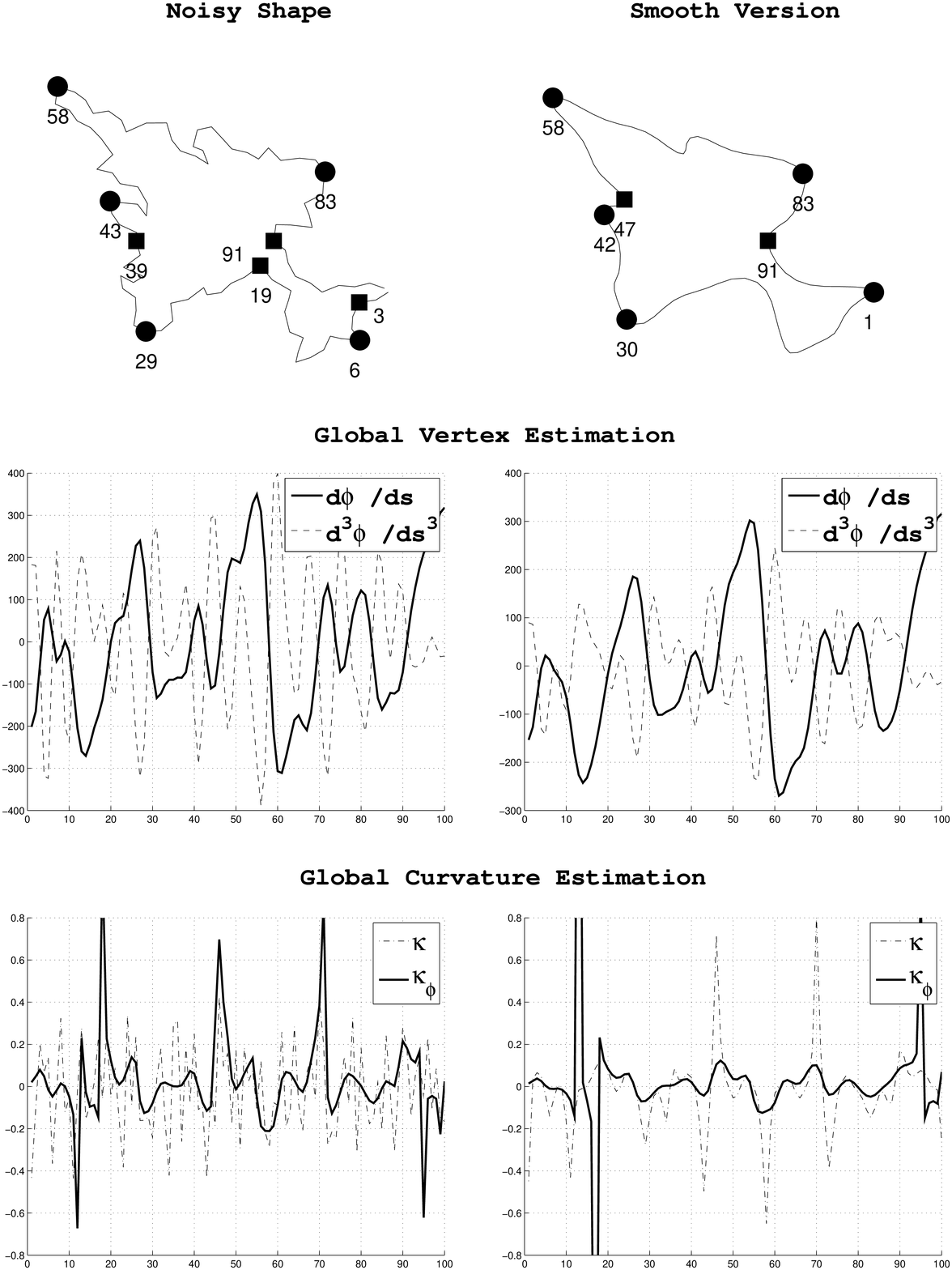}
\caption{Curvature estimation and \textit{Global Vertex} localization of a typical KIMIA silhouette. At the top row of plots, solid circles and squares are curvature local minima and maxima respectively, identified by the proposed method for smooth and noisy versions of the same KIMIA silhouette. Compare these to the stars and diamonds identified for the same shapes by the LAII method in Fig.(\ref{fig:comp2}). At the middle row of plots, $\dot\phi$ and $\tdot\phi$ are shown, the collocation of which at a zero crossing, defines vertices according to the proposed method. At the bottom row of plots we see the differential curvature ($\kappa$) and the curvature estimator $\kappa_{\phi}$(based on equation (\ref{eq:L15aaa})), valid only at points where $A\ne 0$, for both noisy and smooth shapes. Notice also the robustness of the proposed estimator in the noisy case where the differential curvature has collapsed.}
\label{fig:comp1}
\end{center}
\end{figure} 
In this section we validate the above methods with experiments. First we demonstrate the localization of Global Vertices and how it is improved by noise. Then, in (\ref{sec:comp}), in the lack of other \textit{global} method for vertex localization, we compare with LAII\cite{ssoatto}, as the most robust among other local methods. LAII is a \textit{localized} method primarily for noise resistant curvature calculations. By mean of integrals, LAII defines certain 2D circles around every point on the curve. These circles \textit{absorb} noise by providing robust curvature estimations in a certain \textit{localized} sense. If seen as a robust curvature method that uses integrals, LAII is comparable to the proposed herein global method. Vertices however, according to LAII, will be localized as curvature extrema according to the traditional differential definition, whereas in the proposed approach \textit{Global Vertices} are disconnected from curvature, the main contribution in this paper.  In the lack of other global method, comparisons are performed to LAII in both curvature calculations and vertex localization in noisy and smooth shapes. The experiments show that while noise is absorbed and its negative effects are minimized in LAII, the proposed global approach actually benefits from noise in vertex localization. As noise increases LAII's (and any other local method's in that respect) fight is in minimizing the \textit{negative local effects} of noise, while the proposed method gains better performance by exploiting the \textit{positive global effects} of noise. Curvature calculations according to the global method are also not affected significantly by noise. The comparison to LAII reveals that the advantages of the proposed approach in both vertex identification and curvature estimation stem from the global conceptual approach, a fact that turned noise on our side. 
Another important advantage, unique to the proposed \textit{global} method with regard to LAII and other local methods is in differentiating vertices with respect to context. Indeed local methods, as was discussed in the introductory section, have an inherent deficiency to incorporate context. The proposed method on the other hand, demonstrates how context can be conceptualized naturally when global approaches are employed. 


\subsection{Vertex Localization and the Effect of Noise}

\begin{figure}[t]
\begin{center}
\includegraphics[height=6cm]{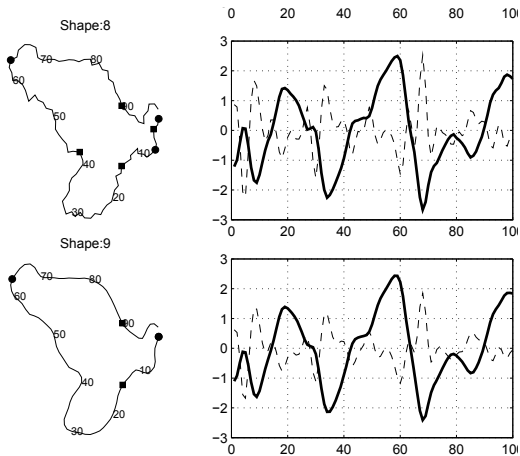}
\caption{Solid line is $\dot\phi$, dashed line is $\tdot\phi$. The method identifies \textit{undisputed} points of extreme curvature (Global Vertices) on the shape's boundary in the co-localization of the above zero crossings. Local maxima curvature points (convex) are marked as solid circles, local minima (concave) are marked as solid squares. \textit{More} points are \textit{correctly} identified in the noisy version.}
\label{fig:small}
\end{center}
\end{figure}

\begin{figure}[t]
\begin{center}
\includegraphics[height=10cm]{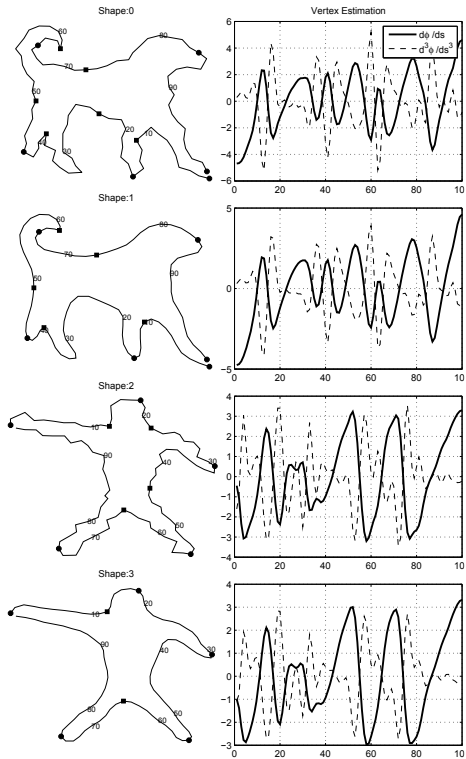}
\caption{Caption as in Fig. \ref{fig:small}.}
\label{fig:S1}
\end{center}
\end{figure}

In Fig.(\ref{fig:V-V3}) a noisy silhouette is compared to its smoothed version. There, we see that the noise is actually improving vertex identification by introducing new vertices, not being identified in the smooth version. The reason this happens as we explained before, is due to the nature of the integral descriptors involved in the proposed method of localizing vertices. In a smoothed curve, a point of maximum curvature may or may not appear as extreme point on the boundary, depending on the location of the point with respect to the rest of the curve.  Point No. 63 e.g. is not identified in the smoothed version since $\dot\phi$, even though it is close, it does not actually achieve zero crossing at a neighborhood of 63. Applying noise in the neighborhood of 63 leads to greater diversity in the tangent directions around 63 and in the noisy version we see $\dot\phi$ finally achieving a \textit{zero crossing} there and 63 correctly being identified as a vertex. This effect of noise is valid only for points that are close of being location extremes on the boundary and is therefore location dependent and has no effect for points that are not well located \textit{globally}, e.g. point 15 where we see that, due to its specific  location on the shape, no matter how much noise we apply to its neighborhood, it will never be identified as a vertex. In Figs. (\ref{fig:small},\ref{fig:S1},\ref{fig:S2}), we also see how vertex identification is improved by inducing noise into various MPEG and KIMIA silhouettes. 

\subsection{Comparison to LAII for Vertex Localization and Curvature Estimation}\label{sec:comp}
A comparison to LAII \cite{ssoatto} for estimating curvature but also for localizing vertices in noisy shapes is presented in this experiment. LAII is a low level descriptor of similar complexity to the proposed method, that generalizes the concept of curvature over the noisy segments of curves. A circle of certain radius is used, centered at each point, and the curvature is calculated as the ratio of the area of this circle that lies in the interior of the closed contour. In the case of zero curvature, e.g. a noisy straight line, half of the disk will lie in the interior of the shape, whereas in the case of infinite curvature this portion will tend to zero or to one depending on the sign of the curvature at this point. LAII therefore uses an integral (area of the circle) to estimate curvature but its essentially a local descriptor.

\begin{figure}[]
\begin{center}
\includegraphics[height=8.8cm]{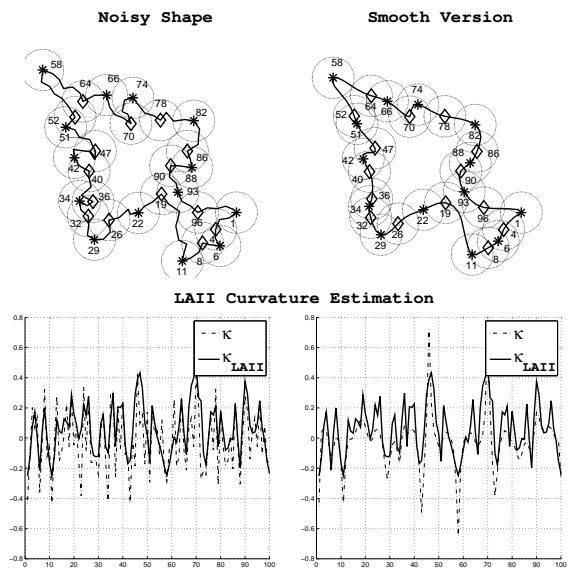}
\caption{Curvature estimation and vertex localization by means of the LAII method\cite{ssoatto}. At the top row plots, stars and diamonds are curvature local minima and maxima respectively, identified by the LAII method from its curvature measurements, both for smooth and noisy versions of the same KIMIA silhouette. Compare this with the circles and squares identified by the proposed method in Fig. (\ref{fig:comp1}). At the bottom row of plots we see the differential curvature ($\kappa$) and the curvature estimator according to LAII($\kappa_{LAII}$). LAII shows high resistance to noise, its curvature measurements are not affected by noise. However, the results in both curvature estimation and vertex localization are more fuzzy and less intuitive in comparison to the ones achieved by the proposed method in Fig.(\ref{fig:comp1}). The circles LAII uses around each identified point, are also plotted for easy visual verification.}
\label{fig:comp2}
\end{center}
\end{figure} 

Our implementation of LAII is as follows: Starting from a binary image of the shape to be encoded, first we extract the boundary. The boundary is then discretized by sampling 100 equally spaced points on it.Then using a circular kernel (constructed as a binary image of a circle of radius 15, as is suggested in \cite{ssoatto}) we convolve the filter with the shape image only at the boundary points. The values of the convolution at each of the boundary points are the values of the LAII estimated curvature at these points. For vertex localization we pick the LAII points of local minima and maxima, marked with stars and diamonds accordingly in Fig. (\ref{fig:comp2}) .

We compare this implementation of LAII to the proposed method of calculating curvature from equation (\ref{eq:L15aaa}), whereas vertex identification is performed as in the previous section by examining  the co-localization of the local extrema of the total distance function $\phi$ with those of its second derivative $\ddot\phi$. For both methods the same extracted contours were used. Many advantages of the proposed method are apparent in this experiment. A comparison of the curvature estimators $\kappa_{\phi}$ and $\kappa_{LAII}$ in Fig.(\ref{fig:comp1}) and Fig.(\ref{fig:comp2}) respectively, reveals better accuracy for the proposed method $\kappa_{\phi}$. The comparison is performed against the differential curvature which, as we see in the figures, collapses in the presence of noise. Also the localization of vetices, marked in the top row of plots as filled squares and circles for the proposed method in Fig.(\ref{fig:comp1}), is more accurate and intuitive than the respective stars and diamonds marked as such by the LAII method in Fig.(\ref{fig:comp2}). In both the  smooth and noisy versions the LAII localization of vertices is often mislead by local boundary formations. The middle row plots in Fig.(\ref{fig:comp1}) reveal the co-localization mechanism of the proposed method.    

%

\begin{figure*}[t]
\makebox[\textwidth][c]{
\begin{minipage}[b]{7cm}
\includegraphics[height=10cm]{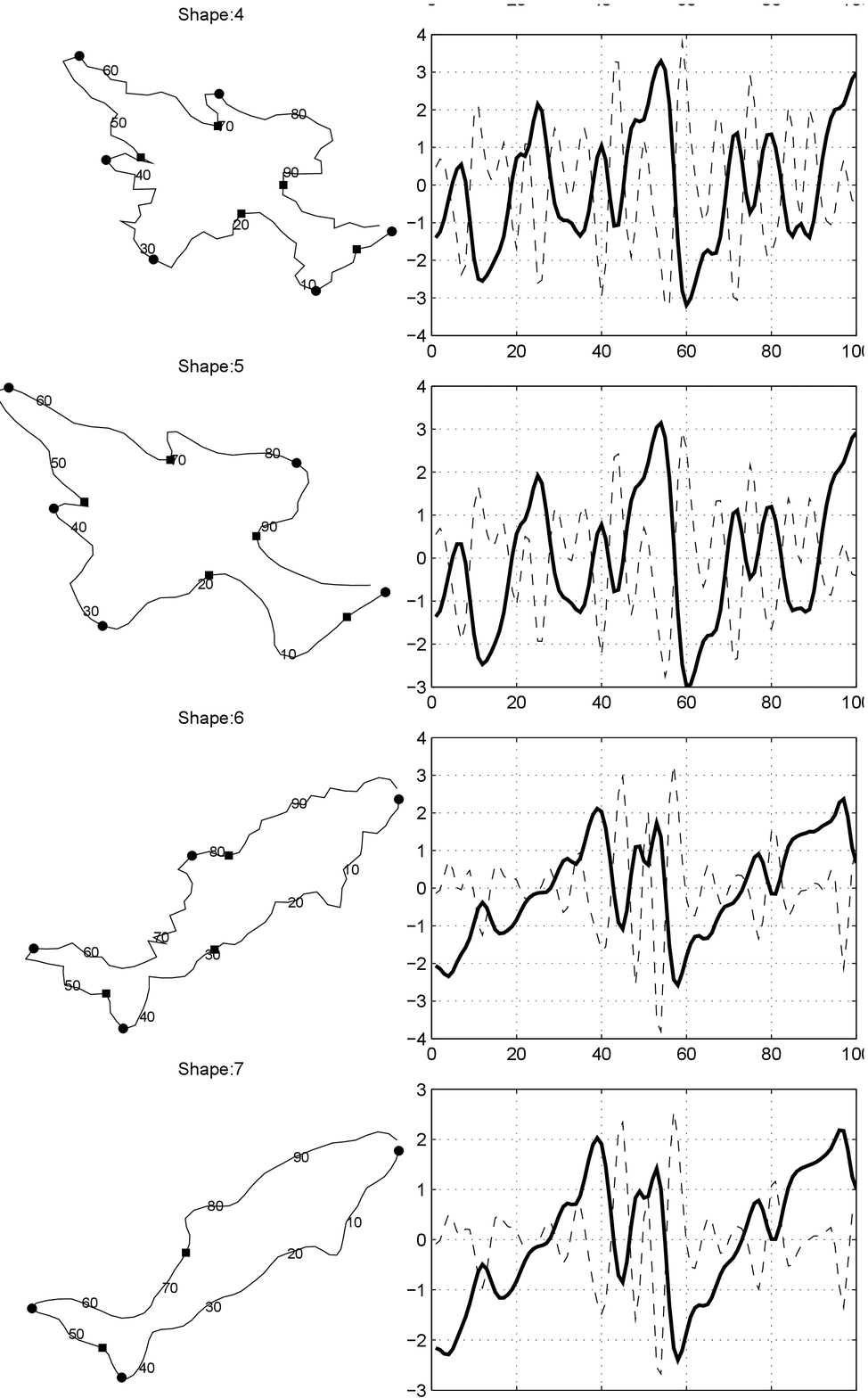}
\end{minipage}
\begin{minipage}[b]{7cm}
\includegraphics[height=10cm]{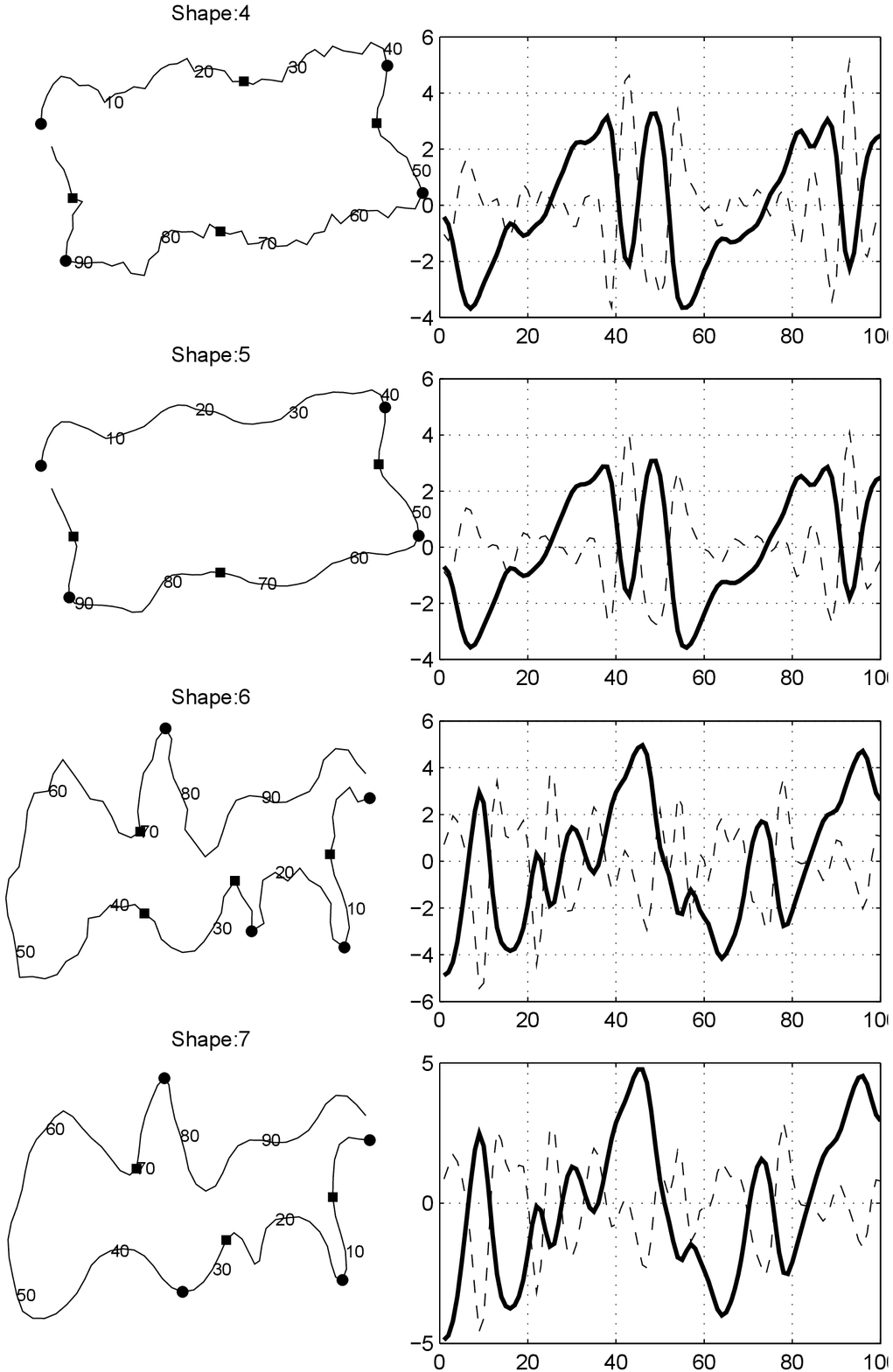}
\end{minipage}
}
\caption{Caption as in Fig. \ref{fig:small}.}
\label{fig:S2}
\end{figure*}

%



\section{Discussion}\label{sec:disc}
The proposed method is attempting to bring in perceptual characteristics to the low level task of vertex identification by combining a robust vertex estimator $\ddot\phi$ and a global position estimator $\phi$. The non trivial local extrema of the former are identified in the \textit{zero crossings} of $\tdot\phi$ and is already a robust vertex estimator, since it is based on the distance to the rest of the curve and not on local to the curve differentials. In the experiments and also according to the theory in the manuscript, one could follow that $\ddot\phi$ can be used as a vertex estimator by itself and vertices could be defined directly on its local extrema (zero crossings of $\tdot\phi$) accordingly. However, although robust in comparison to the differential alternatives, vertices identified this way would be according to the formal mathematical definition, lacking any attempt of differentiation with respect to their importance. This is where $\phi$ plays an important role. 

The introduction of $\phi$ in the process of identifying vertices, provides an additional capability, that of acquiring a \textit{global view} of the locations of the vertex points, introducing the concept of \textit{Global Vertices}, as \textit{undisputed} points of extreme curvature that due to their location at the extremes of the shape (zero crossings of $\dot\phi$) are characterized as perceptually important \footnote{The fact that remote points are perceptually important has been demonstrated in \cite{CVIU}, where shape matching scores in benchmark datasets where improved by permitting correspondences only to \textit{remote} points.}. The key concept behind \textit{Global Vertices} is that extreme location implies extreme curvature (power of location). 
Extreme points on the curve cannot exist without extreme curvature there. The inverse does not hold for all vertices, thus the distinction between Global (according to Proposition \ref{propvertex1}) and regular (according to the traditional differential definition) vertices.  

One can thus say that the \textit{co-location} of the zero crossings of $\tdot\phi$ and $\dot\phi$, that the proposed method uses, can be interpreted as combining extreme curvature with extreme location, thus an attempt to introduce \textit{global perceptual characteristics} to a problem that is natively of a local nature. In the experiments we have used characteristic shapes from the KIMIA and the MPEG datasets in pairs of smooth and noisy versions. Notice that while $\tdot\phi$ zero crossings alone could identify \textit{regular} vertices, the co-location with $\dot\phi$ zero crossings results in picking only the \textit{Global} vertices, the \textit{undisputed} points of extreme location, in an attempt to introduce perceptual \textit{meaningfulness} in the process. Since extreme location is a global feature, adding noise to the shape has no significant effect. In fact, the only result is to introduce a few more extreme points that were very close, but didn't quite appear as such in the smooth version and this is the mechanism behind improving the results by noising. 

One could also notice here that $\ddot\phi$ is used as a vertex estimator and not as a curvature estimator, since curvature is estimated by equation (\ref{eq:L15aaa}) in the manuscript, 
thus the proposed method is calculating vertices \textit{directly}, bypasses curvature  and brings noise on our side. As is shown in the experiments, there are cases of shapes where different noisy versions favor certain vertices. For this reason, the method in practice runs for various noisy scenarios finally considering the \textit{union} of the vertices identified in each of the scenarios. Furthermore, one could control the accuracy of zero crossings co-localization by adjusting the width of the \textit{finitesimal} interval in which a co-localization is assumed to have occurred, achieving this way different degrees of location dependence. Further tuning of the method can be achieved by posing restrictions on the slope of the zero crossings or the magnitude of $\ddot\phi$, resulting in further \textit{meaningful} differentiation among vertices. 

%

\section{Conclusion}\label{sec:conc}

A global definition of curvature, through integrals of global descriptors, instead of differentials, is used to describe a noise enabled method for \textit{meaningful} vertex localization in unknown shapes. The method gives valid results on both smooth and noisy shapes. Furthermore, the results of the method are improved in the presence of noise, thus the concept of \textit{noising} emerges as a paradox. This counterintuitive result was indicated by the theoretical findings and was validated by experiments. A comparison to LAII, a low complexity method of estimating curvature in the presence of noise, reveals the advantages of the proposed approach versus local or localized methods.

{\small
\bibliographystyle{splncs}
\bibliography{egbib}
}
\end{document}